\pdfoutput=1
\documentclass[11pt]{article}

\usepackage{amsmath,amssymb}
\usepackage{amscd,amsthm}
\usepackage{amscd}

\usepackage{multirow}
\usepackage{accents}
\usepackage{hyperref}
\usepackage{color}
\usepackage[usenames,dvipsnames]{xcolor}
\usepackage{graphicx}
\usepackage{enumitem}

\usepackage{algorithm}
\usepackage{algorithmic}

\usepackage{setspace}

\usepackage{graphicx} 

\usepackage{url}
\usepackage{wrapfig}

\usepackage[
pass,
]{geometry}

\newgeometry{top=1in, bottom=1in, left=1in, right=1in}

\newfont{\mycrnotice}{ptmr8t at 7pt}
\newfont{\myconfname}{ptmri8t at 7pt}

\newcommand{\Prob}{\mathbb{P}}

\newcommand{\Image}{\mathrm{Im}}
\newcommand{\fa}{\:\forall\:}
\newcommand{\Fa}{\:\:\forall\:}

\theoremstyle{plain}
\newtheorem{theorem}{Theorem}
\newtheorem{lemma}{Lemma}
\newtheorem{statement}{Statement}

\theoremstyle{definition}
\newtheorem{remark}{Remark}
\newtheorem{problem}{Problem}
\newtheorem{definition}{Definition}

\title{Finite-State Extreme Effect Variable}
\date{December 12, 2019}

\author{
Alexey	Drutsa\thanks{Affiliation: Yandex (16, Leo Tolstoy St., Moscow, Russia, 119021; www.yandex.com).}
}

\begin{document}

\maketitle

\begin{abstract}
	We generalize to the finite-state case the notion of the extreme effect variable $Y$ that accumulates all the effect of a variant variable  $V$  observed in changes of another variable  $X$.
	We conduct theoretical analysis and turn the problem of finding of an effect variable 
	into a problem of a simultaneous decomposition of a set of distributions. 
	The states of the extreme effect variable, on the one hand, are  minimally affected by the variant variable $V$ and, on the other hand, are extremely different with respect to the observable variable $X$.
	We apply our technique to online evaluation of a web search engine through A/B testing and show its utility.
\end{abstract}

\section{Introduction}

Data-driven decision making in many applications turns into understanding or evaluation of 
different variants of a system, that are caused by some variable $V$, through observation of some variable $X$.
For example, such situations occur in A/B testing, the state-of-the-art technique of online evaluation of web services~\cite{2009-DMKD-Kohavi,2014-KDD-Kohavi,2017-TWEB-Drutsa,2018-WSDM-Budylin,2019-SIGIRTut-Drutsa}, where controlled experiments are conducted to detect the causal effect of the system updates on its performance relying on an evaluation metric.

The authors of \cite{2015-KDD-Nikolaev} introduced the notion of a  \emph{binary effect variable} $Y$ in order to compare the observed variable $X$ w.r.t.\ two possible states of the variant variable $V$  (e.g., treatment assignment variable or instrumental variable). Namely, let us assume that $X$ does not directly depend on the variant $V$, but there is only indirect dependency via the latent variable~$Y$, i.e., the following equality of conditional distributions holds: $\Prob(X\mid Y, V)=\Prob(X\mid Y)$. In this way, there is no causal effect \cite{spirtes2000causation,pearl2000causality,lemeire2006causal,claassen2013causal} of the variant variable $V$ that is not conditioned by $Y$. 
In our study, \emph{we generalize} the notion of the effect variable $Y$ to the \emph{multi-state case}.


We conduct thorough theoretical analysis of effect variables with finite number of states and turn the problem of finding of such variable into a distribution decomposition problem.
Since we show that, in general case, the number of possible effect latent variables is infinite, we introduce  the notion of \emph{the extreme effect variable} $\hat{Y}$ (as in the binary case), which simultaneously:
(a) is such effect variable that its states are minimally affected by the variant variable $V$;
(b) is such effect variable that its states are extremely different w.r.t.\ the observable variable $X$.
In addition, we demonstrate  application of this technique to several real examples of quality evaluation of a web search engine.

\section{Related work}
Our  work could be compared with other studies in three aspects.
The first one relates to causal discovery~\cite{spirtes2000causation,pearl2000causality,lemeire2006causal,drton2008lectures,manski2009identification,claassen2013causal,stegle2010probabilistic}, where the main problem consists in finding a causal relations between two or more variables (e.g., a causal graph or chain).
In our study, we state that the causal effect $V \Rightarrow X$ exists and assume that it could be decomposed into  $V \Rightarrow Y \Rightarrow X$. Then, we try to find such latent variable $Y$ in a non-parametric way.
For instance, Balke and Pearl~\cite{balke1997bounds} studied a model where reception of a treatment did depend not only on treatment assignment, but also on some latent variables (factors). In contrast to that study, we consider a model where the treatment assignment variable $V$ affects an observed response $X$ via a latent variable $Y$.
Tishby et al.~\cite{Tishby-Pereira-Bialek-bottleneck} studied the problem of finding a short code for one signal (variable) that preserved the maximum information about another signal (variable). In contrast to our work, they optimized so-called distortion measure to find solution to their problem. Other studies optimized Shanon's  entropy $H[Y]$ (or its conditional variant $H[X|Y]$)  over all possible mediating variables $Y$. In our study, we find an optimal solutions in our problem s.t.\ $\Prob(Y = j \mid V=i )$ differ between the variants $\mathcal{V}$ as little  as possible.

The second group of related studies concerns  decompositions of a distribution into a mixture of several other distributions
(e.g., approximations by a mixture of parametric distributions like in~\cite{von2007multivariate}, or exact methods to decompose into a basis like in~\cite{wang1996wigner}).
To the best of our knowledge, in our study, we consider a novel problem of a simultaneous decomposition of a tuple of distributions into a mixture of other distributions.

The studies on online A/B testing (e.g., \cite{2010-KDD-Tang,2013-KDD-Bakshy,2014-KDD-Kohavi,2014-WWW-Deng,2015-WWW-Drutsa,2018-WSDM-Budylin}) form the third aspect of related work, since we apply our approach to online evaluation of web services.
There is a line of works devoted to evaluation of  different  components of  web services:  
ranking algorithms~\cite{2013-WWW-Song,2015-WSDM-Drutsa,2015-KDD-Nikolaev,2017-TWEB-Drutsa},
the user interface~\cite{2009-IWDMCS-Kohavi,2015-WSDM-Drutsa,2015-KDD-Nikolaev,2017-TWEB-Drutsa},  
etc.
Other works on A/B experiments
address various web user experience:
absence~\cite{2013-WSDM-Dupret}, 
abandonment~\cite{2014-KDD-Kohavi},
engagement~\cite{2015-WSDM-Drutsa,2015-WWW-Drutsa,2015-SIGIR-Drutsa,2015-CIKM-Drutsa}, 
speed~\cite{2014-KDD-Kohavi}, 
periodicity~\cite{2015-WSDM-Drutsa,2015-SIGIR-Drutsa,2017-TWEB-Drutsa,2017-WWW-Drutsa}, etc.
Studies focused on improvement of A/B test metric sensitivity considered: 
utilization of more data from the period before the experiment~\cite{2013-WSDM-Deng,2016-KDD-Poyarkov} and from the experiment period either by learning a linear combination of metrics~\cite{2017-WSDM-Kharitonov} or by predicting a future metric value~\cite{2015-WWW-Drutsa}. Budylin et al.~\cite{2018-WSDM-Budylin} proposed a tool (referred to as linearization) that allowed to efficiently and directly apply all existing sensitivity improvement techniques to ratio metrics (such as CTR). Our study extend the technique of Nikolaev et al.~\cite{2015-KDD-Nikolaev} to multidimensional case. 
More details on A/B testing can be find in surveys like~\cite{2009-DMKD-Kohavi} or in some books on  randomized experiments in general like~\cite{freedman2010statistical,morgan2014counterfactuals} or in   tutorials like~\cite{2018-WWWTut-Budylin,2018-KDDTut-Budylin,2019-SIGIRTut-Drutsa}.


\section{Framework}


In this section, we introduce the core definitions and notations of our study.

\subsection{Effect latent variable}
Let $\Omega$ be a set of random events (e.g., experimental units in A/B testing) and let $\Prob$ denote the probability over them.
Let $X: \Omega \rightarrow \mathcal{X}$ be an observable variable (e.g., the number of visits of a user $\omega\in\Omega$). 
Let $V: \Omega \rightarrow \mathcal{V}$ be a variant variable (e.g., it could represent a variant of a service shown to a user $\omega\in\Omega$).

\begin{definition}
	\label{def:ELV}
	A variable $Y$ is called an \emph{effect (latent) variable} (of $X$ relative to $V$) if
	\begin{equation}
	\label{eq:def_ELV}
	\Prob(X \mid V , Y) = \Prob(X \mid Y).
	\end{equation}
\end{definition}
The region of values of the variable $Y$ is referred to as $\mathcal{Y}$.

\begin{remark}
	\label{remark:ELV}
	Note that Identity~(\ref{eq:def_ELV}) is equivalent to the following one
	\begin{equation}
	\label{eq:X_indep_V_wrt_ELV}
	\Prob(X , V \mid  Y) = \Prob(X\mid Y) \Prob(V\mid Y),
	\end{equation}
	meaning that the variables $X$ and $V$ are independent under condition of the variable $Y$ in  terms of causal theory \cite{spirtes2000causation,claassen2013causal}. In other words, the variable contains all the effect of $V$ observed in the variable~$X$.
\end{remark}

In this study, we consider the case when the regions $|\mathcal{V}|$ and $|\mathcal{Y}|$ are finite  and   are of equal size $|\mathcal{V}| = |\mathcal{Y}| = K < \infty$. These limitations corresponds to practical usage of this theory considered in Section~\ref{sec_PracExamples}. However, the other cases could be studied in future work.


\subsection{Notations and basic definitions}

In our work, we widely use the following well-known terms.
\begin{definition}
	\label{def:PermutMatrix}
	A matrix $Q$ is  a \emph{permutation matrix} \cite{lancaster1985theory,axelsson1996iterative}  if it has exactly one nonzero entry $1$ in each row and each column.
\end{definition}

\begin{definition}
	\label{def:RightStoch}
	A matrix $Q=[q^{j}_{i}]_{i,j=1}^K$ is  a \emph{(right) stochastic matrix} \cite{lancaster1985theory,axelsson1996iterative}  if:
	\begin{itemize}
		\vspace{-0.25cm}
		\item each of its elements is non-negative: $q^{j}_{i} \ge 0 \fa i\in\overline{1,K} \fa j\in\overline{1,K}$\footnote{We refer to the numeration set $\mathbb{Z}\cap[n,N]$ as  $\overline{n,N}$ for any $n,N\in\mathbb{Z}, n < N$.};
		\vspace{-0.25cm}
		\item the sum of all elements in each row is equal to $1$: $\sum_{j=1}^K q^{j}_{i} = 1 \fa i\in\overline{1,K}$.
	\end{itemize}
\end{definition}

\begin{definition}
	\label{def:Distribution}
	A function $g:\mathcal{Z}\rightarrow \mathbb{R}$ is a \emph{distribution} over an enumerable set $\mathcal{Z}$ if :
	\begin{itemize}
		\vspace{-0.215cm}
		\item its values are non-negative: $g(z) \ge 0 \fa z\in\mathcal{Z}$;
		\vspace{-0.25cm}
		\item the sum of all its values over $\mathcal{Z}$ is equal to $1$: $\sum_{z\in\mathcal{Z}} g(z) = 1$.
	\end{itemize}
\end{definition}

From here on in the paper we use shorter notations for distributions: 
$$
\Prob_A(a) \stackrel{def}{=} \Prob(A=a), \quad \Prob_{A \mid B}(a \mid b) \stackrel{def}{=} \Prob(A=a \mid B=b), \quad \hbox{etc}.
$$


\subsection{Distribution decomposition}

Let $d_v(x) = \Prob_{X \mid V}(x\mid v), x\in\mathcal{X}, v\in\mathcal{V},$ for a discrete variable $X$, then the function $d_v:\mathcal{X}\rightarrow \mathbb{R}$ satisfies Definition~\ref{def:Distribution}.
From here on in this section  we consider this case of a discrete variable $X$ (i.e., $\mathcal{X}$ is an enumerable set)\footnote{However, all the described theory could be translated to the case of non-discrete variable $\mathcal{X}$, where the density function of the distribution $\Prob_{X \mid V}$ should be used as $d_v(x)$. This is a good direction for  future work.}
Then, having an effect variable $Y$, we decompose $d_v$  as follows:
\begin{equation}
\label{eq:DD_derive}
\begin{split}
d_v(x) &= \sum\limits_{y\in\mathcal{Y}}\Prob_{X,Y\mid V}(x, y \mid  v) 
= \sum\limits_{y\in\mathcal{Y}}\Prob_{Y \mid V}(y \mid v) \Prob_{X \mid Y, V} (x \mid y, v) \\
&= \sum\limits_{y\in\mathcal{Y}}\Prob_{Y \mid V}(y \mid v) \Prob_{X \mid Y} (x \mid y), \quad x\in\mathcal{X}, v\in\mathcal{V}.
\end{split}
\end{equation}
The last equality is valid, since the identity in Eq.~(\ref{eq:def_ELV}) holds for the effect variable $Y$. Let $p^y_v  = \Prob_{Y \mid V}(y \mid v)$ and $f_y(x) = \Prob_{X \mid Y} (x \mid y)$. Then, the decomposition in Eq.~(\ref{eq:DD_derive}) could be written as follows:
\begin{equation}
\label{eq:DD_base}
d_v(x) =  \sum\limits_{y\in\mathcal{Y}}p^y_v f_y(x) \Fa x\in\mathcal{X} \Fa  v\in\mathcal{V}.
\end{equation}

Since the sets $\mathcal{V}$ and $\mathcal{Y}$ are finite (i.e., $\mathcal{V} = \{v_i\}_{i=1}^K$ and $\mathcal{Y} = \{y_j\}_{j=1}^K$), we define the square matrix $P$, the vectors $\mathbf{d}(x), x\in\mathcal{X}$, and the vectors $\mathbf{f}(x), x\in\mathcal{X}$ as follows:
\begin{equation}
\label{eq:DD_Pmatrix_DFcols}
P = \left[p^{y_j}_{v_i}\right]_{i,j=1}^K, \qquad \mathbf{d}(x)=\left[d_{v_1}(x),..,d_{v_K}(x)\right]^\top, \quad \mathbf{f}(x) = \left[f_{y_1}(x),..,f_{y_K}(x)\right]^\top, x\in\mathcal{X}.
\end{equation}
Then, the decomposition in Eq.(\ref{eq:DD_base}) could be written in the following matrix form:
\begin{equation}
\label{eq:DD_matrix}
\mathbf{d}(x) = P \mathbf{f}(x), \qquad x\in\mathcal{X}.
\end{equation}
Note, that $P$ is a \emph{right stochastic matrix} according to Definition~\ref{def:RightStoch}.





\section{Existence and properties of an effect variable}


In this section, we translate the problem of finding of an effect variable into a non-linear problem of  distribution decomposition described below.

\subsection{Distribution decomposition problem}
Since $|\mathcal{V}|=|\mathcal{Y}|=K$, we assume that $\mathcal{V}=\mathcal{Y}=\overline{1,K}$ without loss of generality. One translates the problem of finding an effect variable $Y$ to the problem of finding a distribution decomposition,
which is formalized as follows. 
\begin{problem}
	\label{prob:DDproblem}
	Having $K$ \emph{source distributions}  $d_i, i\in\overline{1,K},$ over $\mathcal{X}$, find a $K\times K$ matrix $P$  and $K$ function $f_j: \mathcal{X}\rightarrow\mathbb{R}, j\in\overline{1,K},$ so that the following conditions holds:
	\begin{enumerate}
		\vspace{-0.15cm}
		\item Identity~(\ref{eq:DD_matrix}) holds for  $\mathbf{d}=\left[d_1,..,d_K\right]^\top$ and $\mathbf{f}=\left[f_1,..,f_K\right]^\top$;
		\vspace{-0.15cm}
		\item $P$ is a stochastic matrix (as in Def.~\ref{def:RightStoch});
		\vspace{-0.15cm}
		\item $f_j, j\in\overline{1,K}$, are distributions over $\mathcal{X}$ (as in Def.~\ref{def:Distribution}).
	\end{enumerate}
\end{problem}

If a pair $(P, \mathbf{f})$ is a solution to this problem, then  $P = [p_i^j]_{i,j,=1}^{K}$ is referred to as \emph{the mixture matrix}, the scalar values $p^j_i, i,j\in\overline{1,K}$ are referred to as \emph{the mixture coefficients}, and the tuple $\mathbf{f}$ of $K$ distributions $f_j, j\in\overline{1,K},$ is referred to as \emph{the decomposition basis}.

\begin{remark}
	\label{remark:TrivSolution}
	Note that Problem~\ref{prob:DDproblem} has always $K!$ trivial solutions. The mixture matrix $P$ of such solution is a permutation one and the decomposition basis is $f_i(x) = d_{\sigma_P(i)}(x), i=1,..,K$, where $\sigma_P$ is the permutation of $\overline{1,K}$ w.r.t. the matrix $P$.
	These trivial solutions correspond to the case when the variant variable $V$ is the effect variable $Y$. 
\end{remark}

\begin{statement}
	\label{state:DindepPinvert}
	Let the source distributions $\{d_i\}_{i=1}^K$ be linearly independent, then the solution mixture matrix $P$ is invertible. 
\end{statement}
\begin{proof}
	Let us assume the contrary, namely the image  of the linear operator $P$ has the dimension lower than K, i.e., $\dim\Image P < K$. Then, the linear capsule of $\{d_i\}_{i=1}^{K}$  has no greater dimension, i.e., $\dim\left\langle d_v\right\rangle_{i=1}^{K} \le \dim\Image P < K$. Thus, the distributions  $\{d_i\}_{i=1}^K$ are linearly dependent. We come to a contradiction.
\end{proof}

\begin{remark}
	\label{remark:DindepPinvert}
	Note that the reverse statement is invalid: according to Remark~\ref{remark:TrivSolution} the trivial solution of Problem~\ref{prob:DDproblem} exists always (even in the case of linearly dependent $\{d_i\}_{i=1}^K$). Its mixture matrix $P$ is the identity one, which is invertible.
\end{remark}

Let $(P,\mathbf{f})$ be a solution of Problem~\ref{prob:DDproblem} and $P$  be an invertible matrix, then the following identity holds for the decomposition basis $\mathbf{f}(x)$ (see the first condition of  Problem~\ref{prob:DDproblem}):
\begin{equation}
\label{eq:Fx_from_Dx}
\mathbf{f}(x) = P^{-1} \mathbf{d}(x), \qquad x\in\mathcal{X}.
\end{equation}

Next, we consider the following problem, which concerns finding of the mixture matrix (coefficients) only.
\begin{problem}
	\label{prob:Pproblem}
	Having $K$ \emph{source distributions}  $d_i, i\in\overline{1,K},$ over $\mathcal{X}$, find a $K\times K$ matrix $P$, so that the following conditions holds:
	\begin{enumerate}
		\vspace{-0.15cm}
		\item $P$ is an \emph{invertible} stochastic matrix;
		\vspace{-0.15cm}
		\item the inequalities
		\begin{equation}
		\label{eq:Pproblem}
		f_j(x, P, \mathbf{d}) \ge 0  \Fa  j\in\overline{1,K}  \Fa  x\in\mathcal{X}
		\end{equation}
		hold for
		\begin{equation}
		\label{eq:Pproblem_F}
		\mathbf{f}(x, P, \mathbf{d}) = P^{-1} \mathbf{d}(x).
		\end{equation}
	\end{enumerate}
\end{problem}

\begin{lemma}
	\label{lem:PprobToDDprob}
	Let $P$ be a solution of Problem~\ref{prob:Pproblem}, then $(P,\mathbf{f})$, where $\mathbf{f}(x) = \mathbf{f}(x, P, \mathbf{d})$ from Eq.~(\ref{eq:Pproblem_F}), is a solution of Problem~\ref{prob:DDproblem}. 
\end{lemma}
\begin{proof}
	The identity in Eq.~(\ref{eq:DD_matrix}) holds since $P \mathbf{f}(x) = P \mathbf{f}(x) = P P^{-1} \mathbf{d}(x) = \mathbf{d}(x)$. The condition on $P$ in Problem~\ref{prob:DDproblem} holds since $P$ is a stochastic matrix as a solution to Problem~\ref{prob:Pproblem}. 
	Finally,  for $P^{-1}=[{\tilde p}_j^i]_{j,i=1}^K$, the identity $P^{-1}\mathbf{1}^\top = \mathbf{1}^\top$ holds, since $P$ is a stochastic matrix. Then,
	\begin{equation}
	\label{eq:PprobToDDprob_proof}
	\sum\limits_{x\in\mathcal{X}}f_j(x, P, \mathbf{d}) 
	= \sum\limits_{x\in\mathcal{X}}\sum\limits_{i=1}^K {\tilde p}_j^i d_i (x) 
	= \sum\limits_{i=1}^K {\tilde p}_j^i \sum\limits_{x\in\mathcal{X}} d_i (x)
	= \sum\limits_{i=1}^K {\tilde p}_j^i \cdot 1 = 1, 
	j \in\overline{1,K}.
	\end{equation}
	From Eq.~(\ref{eq:PprobToDDprob_proof}) and the inequalities~(\ref{eq:Pproblem})  we conclude that  $f_j(x, P, \mathbf{d}), j \in\overline{1,K},$ are distributions, i.e., the last condition in Problem~\ref{prob:DDproblem} holds.
\end{proof}

Thus, from Statement~\ref{state:DindepPinvert}, Eq.~(\ref{eq:Fx_from_Dx}), and Lemma~\ref{lem:PprobToDDprob}, we infer that if the source distributions $\{d_i\}_{i=1}^K$ are linearly independent,  Problem~\ref{prob:DDproblem} and Problem~\ref{prob:Pproblem} are equivalent w.r.t.\ the set of their solutions.


\subsection{Solutions to the decomposition problem}

Now, we are ready to find the set of all mixture matrices that are solutions to Problem~\ref{prob:Pproblem}. This set is denoted by $\mathfrak{P}$ and the injection of its coefficients in $\mathbb{R}^{K^2}$ is denoted by $\mathcal{P}$. First of all, we state the properties of these sets in the following lemma.
\begin{theorem}
	\label{th:PprobSolutionProp}
	Let the source distributions $\{d_i\}_{i=1}^K$ be linearly independent, then:
	\begin{enumerate}
		\vspace{-0.15cm}
		\item $\mathfrak{P}$ is not empty and contains at least $K!$ trivial solutions that are all permutation matrices (corresponding decomposition bases are permuted source distributions $\{d_i\}_{i=1}^K$);
		\vspace{-0.15cm}
		\item $\mathfrak{P}$ does not contain any degenerate matrix (i.e., $\det P\neq 0 \fa P\in\mathcal{P}$);
		\vspace{-0.15cm}
		\item $\mathfrak{P}$ is closed w.r.t.\ multiplication by a permutation matrix from the right;
		\vspace{-0.15cm}
		\item $\mathcal{P}$ is a closed manifold of the dimension $K^2 - K$ in general case;
	\end{enumerate}
\end{theorem}
\begin{proof}
	The property~1 follows from Remark~\ref{remark:TrivSolution}.
	The property~2 follows from the definition of a solution of Problem~\ref{prob:Pproblem}.
	Next, let $Q$ is any  permutation matrix and $P' = PQ$. Then, $P'$ is a stochastic matrix and $\mathbf{f}(x, P', \mathbf{d}) = {P'}^{-1} \mathbf{d}(x) = Q^{-1}P^{-1} \mathbf{d}(x) = Q^{-1} \mathbf{f}(x, P, \mathbf{d}) \ge 0$, since it is a permutation (inverse to $Q$) of $f_j(x, P, \mathbf{d}), j\in\overline{1,K}$, that are non-negative.
	
	Finally, the restrictions on the components of a stochastic matrix (Def.~\ref{def:RightStoch}) imply that $\mathcal{P}$ belongs to the intersection of $K$ hyperplanes $\{\sum_jp^{j}_{i}=1\}, i\in\overline{1,K},$ and of $K^2$ closed half-spaces $\{p^{j}_{i} \ge 0\}, i,j\in\overline{1,K}.$   
	For $\det P > 0$ ($\det P < 0$), the inequalities in Eq.~(\ref{eq:Pproblem}) define a closed set in $\mathbb{R}^{K^2}$ for each $x\in\mathcal{X}$. Thus, the region $\mathcal{P}$ is the intersection of all of them, and, hence, $\mathcal{P}$ is also closed, i.e., the property~4 holds. 
\end{proof}

%



\subsection{Summary}
In this section, we translate the problem of finding of an effect variable $Y$ into non-linear Problem~\ref{prob:DDproblem} of the simultaneous decomposition of the distributions $d_i, i\in\overline{1,K}$. Then, the equivalence of this problem to non-linear Problem~\ref{prob:Pproblem} of finding proper stochastic matrix with constraints~(\ref{eq:Pproblem}) is established for the linearly independent distributions $d_i, i\in\overline{1,K}$. Finally, the key properties of solutions to these problems are proved in Theorem~\ref{th:PprobSolutionProp}.


\section{Extreme effect variable}
\label{sec_EEV}

Since Problem~\ref{prob:DDproblem} has a continuum of possible solutions in a general non-degenerated case (see Theorem~\ref{th:PprobSolutionProp}), we introduce the notion of the \emph{optimal distribution decomposition} and the notion of the \emph{extreme effect variable}.


\subsection{Optimal distribution decomposition}
\begin{definition}
	\label{def:EEV}
	A  distribution decomposition $(\hat{P},\hat{\mathbf{f}})$ for Problem~\ref{prob:DDproblem} is called the \emph{optimal distribution decomposition} if its mixture matrix $\hat{P}\in\mathfrak{P}$ has the lowest absolute value of the determinant among all mixture matrices $\mathfrak{P}$, i.e.,
	\begin{equation}
	\label{eq:def_EEV}
	\hat{P} = \underset{P\in\mathfrak{P}} {\mathrm{argmin}} |\det P|.
	\end{equation}
	The corresponding effect variable $\hat{Y}$,  its mixture matrix $\hat{P}$, its mixture coefficients $\{\hat{p}_i^j\}_{i,j=1}^{K}$, and its decomposition basis $\hat{\mathbf{f}}$ are referred to as the \emph{extreme effect variable}, the \emph{extreme mixture matrix}, the \emph{extreme mixture coefficients}, and the \emph{extreme decomposition basis}, respectively.
\end{definition}

\begin{remark}
	\label{remark:EEV_existence}
	Note that the minimum in Eq.~(\ref{eq:def_EEV}) always exists, since $\mathcal{P}$ is a closed bounded region in $\mathbb{R}^{K^2}$ (see Theorem~\ref{th:PprobSolutionProp}) and the function $\det P$ is continuous on $\mathcal{P}$. Moreover, this minimum is non-zero in the case of linearly independent source distributions $d_i, i\in\overline{1,K},$ since each mixture matrix $P\in\mathfrak{P}$ is non-degenerate one, i.e., $\det P \neq 0$ (see Theorem~\ref{th:PprobSolutionProp}).
\end{remark}

The resulting variable $\hat{Y}$ is \emph{extreme simultaneously in two senses}. 
The minimization of the determinant in Eq.~(\ref{eq:def_EEV}) implies that, on the one hand, we try to \emph{find extreme variable $\hat{Y}$, such that its probabilities $\Prob_{\hat{Y} \mid V}(j \mid i) = p^j_i$ differ between the variants $\mathcal{V}$ as little  as possible}.
On the other hand, for any set $\{x_k\}_{k=1}^K\subset\mathcal{X}$ of $K$ elements, we can consider the matrices $D = [d_i(x_k)]_{i,k=1}^K$ and $F = [f_j(x_k)]_{j,k=1}^K$. They are connected by the identity $F=P^{-1}D$ (see Eq.~(\ref{eq:Fx_from_Dx})), and, thus, the identity $|\det F|=|\det D|/|\det P|$ holds. Since the matrix $D$ is a constant one for the given source distributions, this identity infers that the optimal decomposition maximizes the disagreement of the decomposition basis $\mathbf{f}$ (i.e., its \emph{power}).
In terms of the effect variable, we try to \emph{find such effect variable $\hat{Y}$ that has the most different distributions in the basis $\mathbf{f}$}. This agrees well with the intuition that $\{f_j\}_{j=1}^K$ represent $K$ extremely different (absolute) states of an event.



\begin{statement}
	\label{state:EEV_permut}
	Let $(\hat{P}, \hat{\mathbf{f}})$ be an optimal distribution decomposition for Problem~\ref{prob:DDproblem}, 
	then the distribution decomposition $(\widetilde{P}, \widetilde{\mathbf{f}})=(\hat{P}Q, [\hat{f}_{\sigma_Q(j)}]_{j=1}^K)$, obtained by a permutation $\sigma_Q$ of the decomposition basis $\hat{\mathbf{f}}$, is an optimal distribution decomposition as well ($Q$ is the corresponding permutation matrix).
\end{statement}
\begin{proof}
	This statement holds, since $(\widetilde{P}, \widetilde{\mathbf{f}})$ is also a solution to Problem~\ref{prob:DDproblem} (see Theorem~\ref{th:PprobSolutionProp}) and the absolute value of the determinant of a permutation matrix equals to $1$, i.e., 
	\begin{equation*}
	|\det \widetilde{P}| = |\det\hat{P}\det Q| = |\det\hat{P}| = \underset{P\in\mathfrak{P}} {\min} |\det P|.
	\end{equation*}
\end{proof}
In order to clearer understand the meaning of the extreme effect variable, one considers two special cases of the observed and variant variables.

\subsection{The case of $|\mathcal{X}| = K$}
In this case, the observable variable has the same number of states as the variant and the effect variables.
Since $|\mathcal{X}| = K$, we numerate the elements of this set: $\mathcal{X} = \{x_k\}_{k=1}^{K}$ 
and consider the functions $\{d_{i}(x)\}_{i=1}^K$ and $\{f_{j}(x)\}_{j=1}^K$ 
in the following matrix forms:
\begin{equation*}
D = \left[d_{i}(x_k)\right]_{i,k=1}^K \equiv \left[\mathbf{d}(x_1) ...\mathbf{d}(x_K)\right]  
\quad\hbox{and}\quad 
F = \left[f_{j}(x_k)\right]_{j,k=1}^K \equiv \left[\mathbf{f}(x_1)...\mathbf{f}(x_K)\right].
\end{equation*}
Hence, in this case, the restriction on the functions $\{d_{i}(x)\}_{i=1}^K$ ($\{f_{j}(x)\}_{j=1}^K$)  to be distributions (as in Def.~\ref{def:Distribution}) is equivalent 
to the restriction on the matrix $D$ ($F$) to be a stochastic one (as in Def.~\ref{def:RightStoch}). Thus, given the stochastic matrix $D$, the distribution decomposition in Eq.~(\ref{eq:DD_matrix}) could be written in the following form:
\begin{equation}
\label{eq:D_PF}
D = PF,
\end{equation}
and Problem~\ref{prob:DDproblem} could be reformulated as follows: find two stochastic matrices $P$ and $F$, such that Eq.~(\ref{eq:D_PF}) holds.

For any stochastic matrix, it is easy to show that the following lemma holds \cite{lancaster1985theory,axelsson1996iterative}.
\begin{lemma}
	\label{lemma:RstochLess1}
	(a) If $Q=[q^j_i]_{i,j=1}^K$ is a stochastic matrix, then $|\det Q| \le 1$. (b) If $Q=[q^j_i]_{i,j=1}^K$ is a stochastic matrix and  $|\det Q| = 1$, then $Q$ is a permutation matrix.
\end{lemma}

\begin{theorem}
	\label{th:EES_for_X_K}
	In the case of $|\mathcal{X}| = K$, the matrix $D$ is the extreme mixture matrix and the functions $\{f_{j}\}_{j=1}^K,$ defined by $f_{j}(x_k) = \delta^j_k\footnote{$\delta^j_k$ is the Kronecker symbol, i.e., it equals $1$, if $i=j$, and $0$, otherwise.}, j,k\in\overline{1,K},$ are the extreme decomposition basis.
\end{theorem}
\begin{proof}
	First, note that the decomposition basis defined by Eq.~(\ref{eq:Pproblem_F}) is $f_{j}(x_k, D, \mathbf{d}) = \delta^j_k, j,k\in\overline{1,K},$ and its components satisfy the conditions in Eq.~(\ref{eq:Pproblem}). Hence, the matrix $D$ is a mixture matrix for Problem~\ref{prob:Pproblem} (i.e., $D\in\mathfrak{P}$). Second, we will show that it is the extreme mixture matrix.
	From Eq.~(\ref{eq:D_PF}) which holds for any mixture matrix $P\in\mathfrak{P}$, we get the following identity:
	\begin{equation}
	\label{eq:detF_detP_detD}
	|\det F| = |\det P|^{-1} |\det D| \Fa P\in\mathfrak{P}.
	\end{equation}
	Thus, the minimization of the determinant $|\det P|$ is equivalent to the maximization of the determinant $|\det F|$, since $\det D$ is given.
	On the one hand, Lemma~\ref{lemma:RstochLess1} implies $|\det F(P,\mathbf{d})| \le 1$, since the matrix $F(P,\mathbf{d})$, that correspond to the decomposition basis $\mathbf{f}(x, P,\mathbf{d})$,  is a stochastic one. On the other hand, for $P = D$, we have $F(D,\mathbf{d}) = E$, that reaches the maximum (i.e, $\det F(D,\mathbf{d}) = 1$). 
	Hence, the theorem is proved.
\end{proof}

This theorem implies that the observable variable $X$ is the extreme effect variable $\hat{Y}$. Moreover, any extreme effect variable $\hat{Y}$ is equal to the observable variable $X$ up to a proper matching of their states (see Statement~\ref{lemma:RstochLess1}), i.e. $\Prob_{\hat{Y} \mid V} = \Prob_{X\mid V}$.  

\subsection{The binary variable case ($K = 2$)}
\label{subsubsec:binarycase}
We return to the general case of $\mathcal{X}$ and consider now the binary variant variable $V$ (and, thus, any effect variable $Y$ is also binary). This lead us to the situation considered in~\cite{2015-KDD-Nikolaev}. Let $\mathcal{V} = \{A, B\}$ and $\mathcal{Y} = \{0,1\}$, then any mixture matrix $P\in\mathfrak{P}$ and the distribution decomposition in Eq.~(\ref{eq:DD_matrix}) could be written
in the following form:
\begin{equation}
\label{eq:DD_2D}
P = \left(
\begin{array}{cc}
p_A  & 1 - p_A \\
p_B  & 1 - p_B 
\end{array}
\right);
\qquad
\begin{cases}
d_A(x) = p_A f_0(x) + (1-p_A) f_1(x), \\
d_B(x) = p_B f_0(x) + (1-p_B) f_1(x) 
\end{cases}
\Fa  x\in\mathcal{X}.
\end{equation}
The conditions in Eq.~(\ref{eq:Pproblem}) are linear for $\det P > 0$ ($\det P < 0$), namely, they are
\begin{equation}
\label{eq:Pproblem_2D}
f_1(x) = \frac{(1-p_B)d_A(x) -  (1-p_A)d_B(x)}{\det P} \ge 0,
\: 
f_0(x) = \frac{p_A d_B(x) -  p_B d_A(x)}{\det P} \ge 0 \:\Fa x\in\mathcal{X}.
\end{equation}
Hence, the set of all solutions of Problem~\ref{prob:DDproblem} in terms of the set $\mathcal{P}$ on the plane of the mixture coefficients $(p_A, p_B)$
is the union two quadrangles that are centrally symmetric to each other w.r.t.\ the point $(1/2,1/2)$ and are limited in the unit square $[0,1]^2$ by the union of two pairs of intersected half-planes (see Fig.~\ref{img_2d_plane}):
\begin{equation}
\label{eq:solutionAlphaPos_2D}
p_B \ge M p_A \quad \hbox{and} \quad p_B \ge 1 - m(1 - p_A),
\end{equation}
\begin{equation}
\label{eq:solutionAlphaNeg_2D}
p_B \le m p_A \quad \hbox{and} \quad p_B \le 1 - M(1 - p_A),
\end{equation}
where
\begin{equation}
\label{eq:MaxMinDef_2D}
m = \inf\limits_{x\in\mathcal{X}}\frac{d_B(x)}{d_A(x)}\in [0,1), \:\:\:\: M = \sup\limits_{x\in\mathcal{X}}\frac{d_B(x)}{d_A(x)}\in (1,+\infty].
\end{equation}

In the considered case, the minimization of the absolute value of the  determinant of the mixture matrix ($|\det P| \rightarrow \min$) is equivalent to the minimization of the absolute value of the difference between the mixture coefficients ($|p_B - p_A| \rightarrow \min$), since $\det P = p_A (1-p_B) - p_B (1-p_A) = p_A - p_B$. In other word, the extreme effect variable $\hat{Y}$ has such decomposition basis $\hat{\mathbf{f}}$ that supplies the minimal \emph{mixture difference} $\alpha = p_B - p_A$ between the probabilities of the variable $\hat{Y}$ to be in the state $0$ (or $1$) for the variants $A$ and $B$. 
Therefore, for $K=2$, Problem~\ref{prob:Pproblem} is linear both in terms of the objective $\det P$, and in terms of the constraints (\ref{eq:Pproblem_2D})\footnote{This is opposite to the cases with $K>2$, where the problem is non-linear.}. Note that there are only two extreme solutions (see Fig.~\ref{img_2d_plane}), and they could be obtained from each other by the swap: $\alpha \leftrightarrow -\alpha$, $p_A \leftrightarrow 1-p_A$, and $p_B \leftrightarrow 1-p_B$. Thus, from Eq.~(\ref{eq:solutionAlphaPos_2D}), Eq.~(\ref{eq:solutionAlphaNeg_2D}), and Fig.~\ref{img_2d_plane}, one infers that the extreme mixture coefficients have the form (for the case of $\alpha > 0$):
\begin{equation}
\label{eq:EEV_2D}
\hat{p}^+_A = \frac{1-m}{M-m}, \qquad \hat{p}^+_B = \frac{M(1-m)}{M-m}, 
\end{equation}
and the difference $\hat{\alpha}^+ = \frac{(M-1)(1-m)}{M-m}$. Note that the \emph{relative difference} $\beta = \alpha / p_A$ for the extreme case has a very simple form:
$\hat{\beta}^+ =M-1$.


\begin{figure}
	\includegraphics[width=0.7\textwidth]{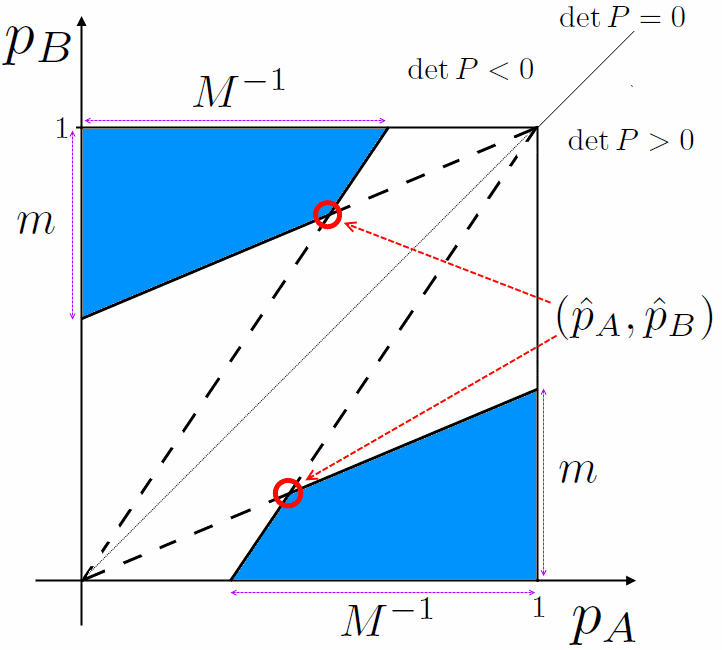}
	\caption{The set $\mathcal{P}$ on the plane $\mathbb{R}^2_{(p_A,p_B)}$ and the extreme mixture coefficients $(\hat{p}_A,\hat{p}_B)$.}
	\label{img_2d_plane}
\end{figure}

\subsection{Summary}
In this section, we introduce the notion of the extreme effect variable $\hat{Y}$ and prove its existence. 
We show that its extreme property is two-fold: 
first, the distributions $\Prob_{\hat{Y} \mid V=v}$ of the variable $\hat{Y}$ w.r.t.\  the variants $v\in\mathcal{V}$ are close each other as much as possible (in terms of the volume of its mixture matrix $\hat{P}$);
and, second, the states of the variable $\hat{Y}$ are extremely different w.r.t. the observable variable $X$ (in terms of the conditional distributions $\Prob_{X \mid \hat{Y}}$).
These intuitions are observed in two considered special cases. The first one relates to the case of the equal number of states of the observable variable $X$ and the variant one $V$ (i.e., $|\mathcal{X}|=|\mathcal{V}|$), where we show that the extreme effect variable $\hat{Y}$ is the observable one $X$.
The second case considers the binary effect variable $Y$ (i.e., $|\mathcal{Y}|=|\mathcal{V}|=2$), where we explain its geometric interpretation and demonstrate the exact formulas for the distributions $\Prob_{\hat{Y} \mid V}$ and $\Prob_{X \mid \hat{Y}}$ for the extreme effect variable $\hat{Y}$. This binary case of the variants is popular in applications: for instance, in A/B testing, the state-of-the-art technique to evaluate a web service~\cite{2010-KDD-Tang,2013-KDD-Bakshy,2014-KDD-Kohavi,2014-WWW-Deng,2015-WWW-Drutsa}, two variants of the service are compared w.r.t.\ a key metric (see the next section).

%


\begin{table}
	\caption{The extreme mixture matrix~$\hat{P}$ and the extreme difference~$\hat{\alpha}^+$ for two A/B tests.}
	\label{tbl_ClModelPerplex}
	\centering
	\begin{tabular}{|c|c|c||c|c|c|}
		\hline
		
		\multicolumn{3}{|c||}{ Exp. \#1 (the number of clicks)}
		& 		\multicolumn{3}{|c|}{ Exp. \#2 (the presence time)} \\
		\hline
		
		 $\hat{P}$ &  $Y = 0$ &  $Y = 1$ 
		& 
				 $\hat{P}$ &  $Y = 0$ &  $Y = 1$ \\
		
		\hline
		
		 $V = A$ &  0.5617 &  0.4384 
		& 
			 $V = A$ &  0.3862 &  0.6138 \\
		\hline
		 $V = B$ &  0.5805 &  0.4195
		& 
				 $V = B$ &  0.4131 &  0.5869 \\
		
		\hline
		\multicolumn{3}{|c||}{ $\hat{\alpha}^+ = 0.0188$, $\hat{\beta}^+ = 0.0336$}
		& 
		\multicolumn{3}{|c|}{ $\hat{\alpha}^+ =0.0269$, $\hat{\beta}^+ = 0.0696$}\\
				
		\hline
	\end{tabular}
\end{table}

\section{Application to online evaluation}
\label{sec_PracExamples}
Online controlled experiments, such as A/B tests, are the state-of-the-art techniques for improving web services based on data-driven decisions and are widely used by many Internet companies~\cite{2010-KDD-Tang,2013-KDD-Bakshy,2014-KDD-Kohavi,2014-WWW-Deng,2015-WWW-Drutsa}. 
The aim of the controlled experiments is to detect the causal effect of the system updates on its performance relying on a user behavior metric that is assumed to correlate with the quality of the system. 
Users are randomly exposed to one of the two variants of a service (the control (A) and the treatment (B), e.g.,  the current production version of the service and its update) in order to compare these variants~\cite{2009-DMKD-Kohavi}.

The experiment's objective is quantitatively measured by \emph{an evaluation metric} $\mu$ (also known as the online service quality metric, the Overall Evaluation Criterion (OEC), etc. \cite{2009-DMKD-Kohavi}), which is usually the average value $\mu = \mathrm{avg}_{\Omega}X$ of a scalar measure $X(\omega)$ (a key metric) over the events (entities) $\omega\in\Omega$, e.g., users.
Thus, for each user group $v \in \{A,B\}$,  we have the distribution $d_v$ of the measure $X$ over the experimental units $\Omega_v$. Then, the average values $\mu_v = \mathrm{avg}_{\Omega_v}X$, $v \in \{A,B\}$ are used as OEC and their difference $\Delta = \mu_B - \mu_A$ is calculated.
Since, the set of variants $\mathcal{V}=\{A,B\}$, the observable variable $X$, and its empirical conditional distributions $d_v(x), x\in\mathcal{X}, v\in\mathcal{V},$ are given, we are able to apply the technique described in Section~\ref{subsubsec:binarycase} in order to find the extreme effect variable and the optimal decomposition of distributions $\{d_A, d_B\}$.

In our study, we consider two popular engagement metrics $X$ of user activity~\cite{2010-CHI-Rodden}: the \emph{number of clicks} made by a user and the \emph{presence time} of a user, that are defined in the same way as in \cite{2007-WebKDD-Jansen,2013-WWW-Song,2013-CIKM-Lehmann,2015-WSDM-Drutsa,2015-WWW-Drutsa}. For each metric, we consider a large-scale A/B experiment conducted on \emph{real users} of Yandex\footnote{\url{https://www.yandex.com}}, one of the popular web search engines. Each of these A/B tests has been designed to evaluate a \emph{noticeable deterioration} of the user interface of the service, has lasted \emph{two weeks} and has affected at least hundreds of thousands of users. 

\begin{figure}
	\includegraphics[width=\textwidth]{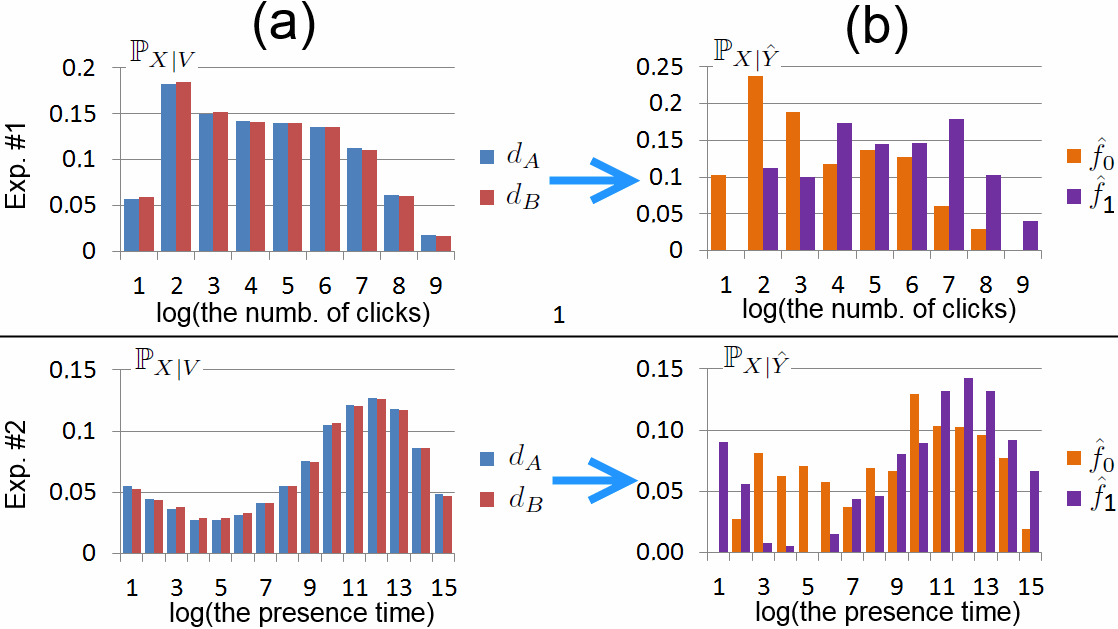}
	\caption{The source distributions $\{d_A, d_B\}$ and the extreme decomposition basis $\{f_0, f_1\}$ for two A/B experiments.}
	\label{img_exp_distr}
\end{figure}

The treatment effect of the first A/B experiment is detected by the decrease of the number of clicks per user by $1.9\%$ (i.e., $\Delta/\mu_A =0.019$), while the treatment effect of the second one is detected by the decrease of the presence time per user by  $1.3\%$ (i.e., $\Delta/\mu_A =0.013$)\footnote{The differences $\Delta$ in both experiments are significant with p-value of \emph{two-sample t-test} lower than $0.05$ (the state-of-the-art threshold \cite{2009-DMKD-Kohavi,2013-WSDM-Deng,2013-WWW-Song,2014-KDD-Kohavi,2015-WSDM-Drutsa,2015-WWW-Drutsa}).}.
In Fig.~\ref{img_exp_distr}(a), we present the distributions $d_A$ and $d_B$ for the metrics of these experiments\footnote{The values of the metrics are logarithmically transformed, multiplied by a fixed random constant and, then, are binarized in order to obtain  discrete distributions for the observable variable $X$. The random constant is hidden for confidentiality reasons.}.
Then, we apply our extreme effect variable approach (see Eq.~(\ref{eq:Pproblem_2D}) and Eq.~(\ref{eq:EEV_2D})), which results in the extreme mixture matrix $\hat{P}$ and the extreme decomposition basis $\{\hat{f}_0, \hat{f}_1\}$, that are presented in Table~\ref{tbl_ClModelPerplex} and Fig.~\ref{img_exp_distr}(b), respectively, for each A/B experiment.

We see that in both experiments the distribution $\hat{f}_0$ corresponds to the extreme state $\hat{Y}=0$ that has lower values of $X$ than the extreme state $\hat{Y}=1$: 
in Exp.~\#1 (in Exp.~\#2), $\hat{f}_0$ corresponds to the state, where users have lower number of clicks (presence time) than in $\hat{f}_1$.
In both experiments, the fraction of the extreme state $\hat{Y}=0$ is increased in the treatment variant B of the service w.r.t.\ the control one (see $\hat{\alpha}^+$ in Table~\ref{tbl_ClModelPerplex}). Hence, we can conclude that the treatment effect in both experiments is negative, since the fraction of the negative extreme state $\hat{Y}=0$ is increased by $\hat{\beta}^+$ percents (by $3.36\%$ for Exp.~\#1 and by $6.96\%$ for Exp.~\#2). This coincides with the conclusion based on the difference of the mean values of the key metrics ($\Delta/\mu_A$).

Summarizing,  \emph{our approach provides more additional information} that allows us to understand more clearly how exactly user behavior, observed through the variable $X$, differs between the variants in an A/B experiment.
First, the approach explains the difference between the observed distributions $d_A$ and $d_B$ \emph{through description of the extreme latent states, such that their mixture is minimally affected by the treatment effect} of the A/B experiment (see Def.~\ref{def:EEV}).
Second, our approach \emph{quantifies this whole (total) difference between the distributions in one scalar metric} $\hat{p}_v, v\in\{A,B\}$, whose relative difference $\hat{\beta}^+$ between the variants is noticeably higher than the one for the mean value (i.e., $\Delta/\mu_A$), since the extreme effect variable $\hat{Y}$ is aware of all the  treatment effect (see Def.~\ref{def:ELV} and~\ref{def:EEV}) observed in the whole conditional distribution $\Prob_{X \mid V}$ of the variable $X$, not only in its mean value.


\section{Conclusions}
In our study, we generalized the notion of the effect variable to the (multidimensional) finite-state case.
We translated the problem of finding an effect variable to the simultaneous decomposition of the conditional distributions of the observable variable under the states of the variant variable. We conducted  theoretical analysis of these problems and their solutions. 
We applied our approach to online evaluation of a web search engine through A/B testing and showed its utility by providing clear additional intuition about the evaluation criterion.


\subsection*{Acknowledgments}
I would like to thank  Gleb Gusev who discussed this study with me and inspired me to public it.

\nocite{2015-WSDM-Drutsa}
\nocite{2015-WWW-Drutsa}
\nocite{2015-SIGIR-Drutsa}
\nocite{2015-KDD-Nikolaev}
\nocite{2015-CIKM-Drutsa}
\nocite{2016-KDD-Poyarkov}
\nocite{2017-WSDM-Kharitonov}
\nocite{2017-TWEB-Drutsa}
\nocite{2017-WWW-Drutsa}
\nocite{2018-WSDM-Budylin}
\nocite{2018-WWWTut-Budylin}
\nocite{2018-KDDTut-Budylin}
\nocite{2019-SIGIRTut-Drutsa}

%

\nocite{balke1997bounds}
\nocite{manski2009identification}
\nocite{drton2008lectures}

\nocite{Lauritzen-sufficiency-and-prediction}
\nocite{Lauritzen-extreme-point-models}
\nocite{Salmon-1971}
\nocite{Shalizi-Crutchfield-bottleneck}
\nocite{Tishby-Pereira-Bialek-bottleneck}

\bibliographystyle{plain}
\bibliography{2019-arxiv-ddm}

\end{document}